\documentclass[manuscript,screen,review]{article}
\usepackage{geometry}                		
\usepackage{amssymb}
\usepackage{amsmath}
\usepackage{amsthm}
\usepackage[colorlinks=true, urlcolor=blue, linkcolor=red]{hyperref}
\usepackage{graphicx}
\usepackage{subcaption}

\newtheorem{prop}{Proposition}[section]

\title{General Post-Processing Framework for Fairness Adjustment of Machine Learning Models}
\author{Léandre Eberhard, Nirek Sharma, Filipp Shelobolin, Aalok Ganesh Shanbhag}

\begin{document}
\maketitle

\section{Abstract}
As machine learning increasingly influences critical domains such as credit underwriting, public policy, and talent acquisition, ensuring compliance with fairness constraints is both a legal and ethical imperative. This paper introduces a novel framework for fairness adjustments that applies to diverse machine learning tasks, including regression and classification, and accommodates a wide range of fairness metrics. Unlike traditional approaches categorized as pre-processing, in-processing, or post-processing, our method adapts in-processing techniques for use as a post-processing step. By decoupling fairness adjustments from the model training process, our framework preserves model performance on average while enabling greater flexibility in model development. Key advantages include eliminating the need for custom loss functions, enabling fairness tuning using different datasets, accommodating proprietary models as black-box systems, and providing interpretable insights into the fairness adjustments. We demonstrate the effectiveness of this approach by comparing it to Adversarial Debiasing, showing that our framework achieves a comparable fairness/accuracy tradeoff on real-world datasets.

\section{Introduction}

In machine learning, the primary goal when building a model is to optimize a given performance metric on a dataset. However, in many fields, such as credit underwriting, public policy, talent acquisition, and others, it is necessary to satisfy other fairness constraints in addition to the performance metric for which the model is optimized. Often, the performance metric is at odds with the fairness metric and optimizing for performance alone can yield a model that does not satisfy the fairness constraints. Numerous works, such as \cite{equalopportunity,agarwal2018reductionsApproach,donini2018ERMfairness,zafar2017fairnessconstraintsmechanismsfair}, have formally studied these fairness requirements, noting that purely optimizing a baseline loss may lead to disparate outcomes for different groups. In practical applications like credit analysis, public policy, or job hiring, large models are frequently used, and sometimes neural networks or boosted trees can outperform each other depending on data characteristics \cite{mcelfresh2024neuralnetsoutperformboosted}.

 A variety of open-source toolkits have emerged to facilitate bias detection and mitigation, such as the IBM AI Fairness 360 Toolkit \cite{aif360-oct-2018}, which has implemented many existing methods. In academic benchmarks, the UCI repository \cite{uci} and the COMPAS dataset \cite{compas} have been commonly used to demonstrate bias or measure fairness improvements.

\subsection{Related Work}

Traditional fairness bias mitigation strategies can be classified into three classes based on the stage of the model-building process at which they are applied \cite{mehrabi2022surveybiasfairnessmachine}:
\begin{list}{}
\item \textbf{Pre-Processing.} The goal of pre-processing methods is to edit the training data used to fit the model in order to improve the fairness of the resulting model, such as by sampling \cite{5360534, iosifidis2018dealing} or perturbation of the data \cite{wang2019repairingretrainingavoidingdisparate, hajian2012methodology}, however both of these methods may change the true data distribution, which can decrease model performance.
\item \textbf{In-Processing.} In contrast, in-processing methods modify the training algorithm to remove bias during model training time, usually by adding regularization to the loss function \cite{jung2020algorithmicframeworkfairnesselicitation, kim2018fairness, dwork2011fairnessawareness}.  However, these methods are heavily dependent on the learning algorithm and can be difficult or impossible to adapt to different algorithms. Some recent work has proposed reductions or probabilistic approaches to unify these constraints under a single framework \cite{donini2018ERMfairness}. A particularly high-performing method is adversarial debiasing, which uses adversarial learning to debias the model \cite{grari2019fairadversarialgradienttree, celis2019improvedadversariallearningfair, wadsworth2018achievingfairnessadversariallearning, zhang2018mitigatingunwantedbiasesadversarial}.
\item \textbf{Post-Processing.} This class of methods is useful when changing the training data or model algorithm is infeasible. The advantage is that the underlying model can be treated as a black-box and does not need to be retrained. Traditional approaches in this class usually involve reassigning predicted labels \cite{lohia2018biasmitigationpostprocessingindividual, pmlr-v81-menon18a, 6413831}. Recent models have also attempted to adjust final task-specific parts of the model \cite{du2021fairness} or using an attention mechanism to explain and adjust trained models \cite{mehrabi2021attributingfairdecisionsattention}.
\end{list}

We generalize existing in-processing techniques into an "adjuster model" framework that is applied as a post-processing step, utilizing adversarial debiasing to achieve high-performance.
The adversarial debiasing method applies to gradient-based optimization approaches (for instance, backpropagation or tree boosting \cite{ad-xgb}). At each iteration of training our model, a secondary adversary model tries to predict some undesirable part of our model. For example, if we want group parity in our model, then we would set up our adversary to predict group membership as a function of our model prediction. Making this adversary model worse by changing our prediction in effect improves the fairness criteria (group parity) that we have selected. 
Given a primary model with weights $W$ and a primary loss function $L_P$, adversarial debiasing updates $W$ according to
\[
\nabla_W L_P \;-\; \,\nabla_W L_A,
\]
where $L_A$ is the adversary’s loss and $\alpha\in\mathbb{R}$ a tunable parameter controlling the strength of fairness. The sign of the adversary term is negative because the main model aims to maximize the adversary’s loss. In practice, once the main model’s gradient is computed, the adversary is also updated via $\nabla_{W_A}L_A$. Intuitively, a higher adversary loss means it becomes more difficult to infer the protected attribute from the main predictions.

Although we focus on adversarial debiasing, our fairness adjuster framework is also compatible with any approach where the final loss can be written as $L_P + c\,L_A$. An example is the prejudice remover approach in \cite{kamishima2012fairness}, which adds a regularizer to remove prejudice in the classification process. Our post-processing perspective can accommodate such in-processing constraints while letting the baseline model remain a black-box, potentially trained with automated techniques like FLAML \cite{wang2021flaml}.

\subsection{Main Contribution}
This paper proposes a general bias mitigation framework that adapts techniques from the in-processing class of methods, but applied as a post-processing step, which we call the \textbf{fairness adjuster}.
This combines the advantages of in- and post-processing methods.
Our framework allows us to tune a model's predictions for fairness without changing its training data or learning algorithm, while maintaining performance near state-of-the-art in-processing methods.

We demonstrate theoretical bounds on the performance difference between applying a technique as a post-processing step rather than an in-processing step.
In practice, we find that the our fairness adjuster is able to match the performance of in-processing methods almost exactly.
Perhaps the most surprising result we find is that we are able to decompose regular in-processing methods into two steps: the first involving only the performance objective, and the second involving only the predictions from the non-debiased model as well as the fairness constraint.

\section{Method}
For data $\{(x_i,y_i)\}_{i=1}^n$ with $x_i \in \mathbb{R}^d$ and $y_i \in \mathbb{R}$ or $y_i \in \{0, 1\}$, some accuracy loss function $L(X, y) \rightarrow \mathbb{R}$ and some fairness loss function $L_a(X, Y) \rightarrow \mathbb{R}$ define three objectives.

A model $h$ trained with traditional in-processing methods optimizes the following objective jointly for accuracy and fairness.
As noted above, our method applies to a general class of in-processing methods, but from now on we focus on adversarial debiasing:

\textbf{Joint Optimization}
\[
L_{AD} = L\bigl(h(X),\,Y\bigr) + \lambda L_a\!\bigl(h(X)\bigr)
\]

This optimization yields a single model that is jointly trained for accuracy and fairness. In contrast, our proposed fairness adjuster methodology separates the optimization for accuracy and fairness into two distinct steps.

1. \textbf{Baseline Optimization} First, we train the baseline model \( f \) focusing solely on performance:
\[
L_B = L\bigl(f(X),\,Y\bigr)
\]

2. \textbf{Fairness Adjustment} In the second step, we fix the baseline predictions \( \hat{Y}_i = f^*(X) \) from the optimization of \( L_B \) in the first step.
Defining the adjuster model as \( g \), we minimize:
\[
L_{O} = L\bigl(\hat{Y} + g(X),\,\hat{Y}\bigr) + \lambda\, L_a\!\bigl(\hat{Y} + g(X)\bigr)
\]

Notice here that we do not use the true label $Y$ in the loss function in the adjustment step!
This is an important distinction, as it allows us to separately train the adjuster without access to the baseline model's training labels.
In the following section, we bound the difference in performance between the joint training approach and adjuster approach, and show empirically that both approaches are able to achieve a similar fairness-accuracy tradeoff.

\section{Theoretical Properties}
In this section, we derive some properties of the adjuster in comparison with the joint optimization approach.
We focus our attention on two common supervised learning tasks \footnote{
    The method also works with other loss functions as well, for example with survival analysis.
}: 
\begin{enumerate}
\item \textbf{Regression}, where $L(u, y) = \sum_{i=1}^n \bigl(u_i - y_i\bigr)^2
$ is the mean squared error (MSE) loss.
\item \textbf{Binary Classification}, where $L(u, y) = -\sum_{i=1}^n 
\Bigl[
  y_i\,\ln\!\bigl(\sigma(u_i)\bigr)
  \;+\;
  \bigl(1-y_i\bigr)\,\ln\!\bigl(1-\sigma(u_i)\bigr)
\Bigr]$ is the binary cross entropy (BCE) loss.
\end{enumerate}

\subsection{Regression}
The results in this section characterize the change in performance and fairness of the fairness adjuster model. First, we characterize the differnece in performance between the fairness adjuster and the baseline model.
Next, we state results comparing the solutions of the adversarial debiasing and fairness adjuster approaches.
Finally, we connect the change in fairness to the size of the adjustment.

\begin{prop}[Performance loss of adjustment from optimal model]
Let \( L(\hat{y}, y) = \sum_{i=1}^n (\hat{y}_i - y_i)^2 \) denote the Mean Squared Error (MSE) loss. For any perturbation vector \( g \in \mathbb{R}^n \), the change in the MSE loss when \( g \) is added to the prediction \( \hat{y} \) given by
\[
\Delta_{\text{MSE}}(f, g) = L(\hat{y} + g, y) - L(\hat{y}, y),
\]
can be expressed as
\[
\Delta_{\text{MSE}}(f, f + g) = \sum_{i=1}^n g_i^2 + 2 \sum_{i=1}^n (\hat{y}_i - y_i) g_i.
\]
Alternatively, it can be decomposed into:
\[
\Delta_{\text{MSE}}(f, f + g) = L(\hat{y} + g, \hat{y}) + 2 \sum_{i=1}^n (\hat{y}_i - y_i) g_i,
\]
where \( L(\hat{y} + g, \hat{y}) = \sum_{i=1}^n g_i^2 \) represents the MSE between \( \hat{y} \) and \( \hat{y} + g \).
\end{prop}

\begin{proof}
\begin{align*}
\Delta_{\text{MSE}}(f, f + g) &= L(\hat{y}+g, y) - L(\hat{y}, y) \\
&= \sum_{i=1}^n  \Bigl[\bigl(\hat{y}_i + g_i - y_i\bigr)^2  - \bigl(\hat{y}_i - y_i\bigr)^2\Bigr] \\
&= \sum_{i=1}^n g_i^2  \;+\; 2\sum_{i=1}^n (\hat{y}_i - y_i)\,g_i \\
\end{align*}
\end{proof}

If both $g_i$ and $(\hat{y}_i - y_i)$ are uncorrelated and symmetric around zero then the extra term will be centered around zero \footnote{
Naturally, MSE is also minimized when $g_i = (\hat{y}_i - y_i)$.
}.
We can more conservatively bound with a Cauchy-Scharwz as follows: \\
\begin{align*}
\Delta_{\text{MSE}}   &\le\; \sum_{i=1}^n g_i^2 \;+\;2 \sum_{i=1}^n |\,\hat{y}_i - y_i|\,|g_i|
\end{align*}
If $\|g\|$ is small or if $\hat{y}_i \approx y_i$, then $\Delta_{\text{MSE}}$ remains small. 
\\

\begin{prop}\label{prop:MSEbound}
For optimal solutions \( g^* \) optimizing \( L_O \) and \( h^* \) optimizing \( L_{AD} \), assume the following conditions hold:

\begin{enumerate}
\item Equal fairness: $L_a\bigl(h^*(X)\bigr) = L_a\bigl(\hat{Y} + g^*(X)\bigr)$
\item Sufficient model complexity: $\hat{Y} + g^*(X)$ can equal $f(X)$ and $h(X)$ with appropriate parameter choices.
\item Convergence: All models converge to the global minima within their model class.

\end{enumerate}
Then, the change in the objective \( \Delta_{\text{MSE}} \) can be bounded by:
\begin{align*}
\Delta_{\text{MSE}} &= L(\hat{y} + g^*(X)) - L(h^*(X)) \\
& \leq \sum_{i=1}^n  2(\hat{y}_i - y_i)(g^*_i - (h^*_i - \hat{y}_i)) \\
& = 2 (\hat{y} - y)^T(g^* - (h^* - \hat{y}^*))
\end{align*}
\end{prop}

\begin{proof}
Using the assumptions that $f^* + g^*$ minimizes $L_O$, and $h^* - f^*$ is in the same model class,
\begin{align*}
    L_O(f^* + g^*) & = \sum_{i=1}^n {g^*}_i^2 + \lambda L_a\!\bigl(f^* + g^*\bigr) \\
    &\leq \sum_{i=1}^n (h^*_i - f^*_i)^2 + \lambda L_a\bigl(h^*\bigr) \\
    & = \sum_{i=1}^n (h^*_i - f^*_i)^2 + \lambda L_a\bigl(f^* + g^*\bigr),
\end{align*}
where the final equality follows from the equal fairness assumption.
Hence, $\sum_{i=1}^n {g^*}_i^2 \leq \sum_{i=1}^n (h^*_i - f^*_i)^2$ and 
\begin{align*}
\Delta_{\text{MSE}} &= L(\hat{y}+g^*(x)) - L(h^*(x)) \\
\;&=\; \sum_{i=1}^n  \Bigl[\bigl((\hat{y}_i - y_i )+ g^*_i\bigr)^2 -  \bigl((h^*_i - \hat{y}_i) - (y_i - \hat{y}_i)\bigr)^2\Bigr]\\
\;&=\; \sum_{i=1}^n  2(\hat{y}_i - y_i)g^*_i + {g^*}_i^2 - (h^*_i - \hat{y}_i) ^2 - 2(h^*_i - \hat{y}_i)(y_i - \hat{y}_i) \\
\;&\leq\; \sum_{i=1}^n  2(\hat{y}_i - y_i)g^*_i +(h^*_i - f^*_i)^2 - (h^*_i - \hat{y}_i) ^2 - 2(h^*_i - \hat{y}_i)(y_i- \hat{y}_i) \\
\;&=\; \sum_{i=1}^n  2(\hat{y}_i - y_i)g^*_i - 2(h^*_i - \hat{y}_i)(y_i - \hat{y}_i) \\
\;&=\; \sum_{i=1}^n  2(\hat{y}_i - y_i)(g^*_i - (h^*_i - \hat{y}_i))
\end{align*}
\end{proof}

We've bounded the loss in accuracy of the adjuster relative to the adversarial debiasing estimator in terms of the cross product between baseline model's residual and the difference in adjustments to the baseline model's predictions, where we reinterpret $h^* - \hat{y}^*$ as an adjustment to the baseline predictions $\hat{y}^*$, similar to $g^*$. Intuitively we would expect the adversarial debiasing model to cause slightly less inaccuracies holding fairness gains the same, because it has fairness in its objective function. \\

Interpreting this bound, there are three mechanisms that can cause the two methods to have a similar training loss.
\begin{enumerate}
\item The baseline model is very accurate with $\hat{y} \approx y$.
\item The fairness adjuster and adversarial debiasing predictions are close: $g^* \approx h^* - f^*$.
\item The model residuals and difference in adjustments to the base model are orthogonal.
\end{enumerate}
In our results, we observe the third case where the residuals and difference and adjustments are near-orthogonal, suggesting that both the adjuster and adversarial debiasing have a similar fairness-accuracy tradeoff in practice.

Note that this result does not say anything about the individual predictions being close, nor that the two models have the same performance on the test set.
However, in the case of linear regression, we can show a stronger result, where the solutions to both problems do coincide exactly.

\begin{prop}[Equivalence in Linear Regression Case]
If all the models (baseline, adversarial debiasing, adjuster) are fit using Ordinary Least Squares Linear Regression, the fairness loss is convex and $X$ is full rank, then the adjuster model coincides with the adversarial debiasing model.
\end{prop}

\begin{proof}
For the baseline model, we have:
\begin{align*}
\beta_f &= \min_{\beta} L_{B} \\
&= \min_{\beta} \;\sum_{i=1}^n (x_i^\top \beta - y_i)^2 \\
&= \min_{\beta} (X^T \beta - Y)^T (X^T \beta - Y)
\tag{1}
\end{align*}
Which has a matrix solution:
\[
\beta_f = (X^\top X)^{-1}X^\top y
\tag{2}
\]
Next, consider the adversarial debiasing estimator that satisfies:
\[
\beta_a = \min_{\beta} L_{AD} = \min_{\beta} 
\Bigl\{
  \sum_{i=1}^n (x_i^\top \beta - y_i)^2
  \;+\;\lambda L_a\!\bigl(X^\top \beta_a\bigr)
\Bigr\},
\]
By convexity:
\[
\frac{\partial L_{AD}}{\partial \beta}\Big\rvert_{\beta=\beta_a} = 2\,X^\top \bigl(X\beta_a - Y\bigr) \;+\; \lambda\,X^\top \Delta L_a(X\beta_a) = 0
\]
\[
\Rightarrow 2\,X^\top \bigl(X\beta_a - Y\bigr) = -\lambda\,X^\top \nabla L_a(X\beta_a)  \tag{3}
\]
For the \emph{adjuster}, fix $\beta_f$ and define
\[
  L_{O} = 
  \sum_{i=1}^n 
  \bigl(\hat{y}_i + x_i^\top \beta - \hat{y}_i\bigr)^2
  \;+\;
  \lambda L_a\!\bigl(\hat{Y} + X^\top \beta\bigr)
\]
Since $\hat{y}_i=x_i^\top \beta_f$, the first term is $\sum_i (x_i^\top \beta_g)^2$. Thus: 
\[
\frac{\partial L_{O}}{\partial \beta}\Big\rvert_{\beta=\beta_g}  = 
2\,X^\top X\,\beta_g 
\;+\;
\lambda\,X^\top \nabla L_a(X(\beta_f + \beta_g))
\]
Consider $\beta_g = \beta_a - \beta_f$. Then 
\[
\hat{y} + \beta_g X = \beta_f X + (\beta_a - \beta_f) X = \beta_a X
\]
\begin{align*}
\Rightarrow \frac{\partial L_{O}}{\partial \beta}\Big\rvert_{\beta=\beta_g} &= 2\,X^\top X\,(\beta_a - \beta_f) \;+\; \lambda\,X^\top \nabla L_a((X\beta_a)) \\
&= 2\,X^\top X\,\beta_a - 2\,X^\top X\,\beta_f \;+\; \lambda\,X^\top \nabla L_a((X\beta_a)) \\
&= 2\,X^\top X\,\beta_a - 2\,X^\top X\,(X^TX)^{-1}X^TY \;+\; \lambda\,X^\top \nabla L_a((X\beta_a)) \tag{\text{from (2)}} \\
&= 2\,X^\top X\,\beta_a - 2X^TY \;+\; \lambda\,X^\top \nabla L_a((X\beta_a)) \\
&= 2\,X^\top X\,\beta_a - 2X^TY \;-\; ( 2\,X^\top X\,\beta_a - 2X^TY ) \tag{\text{from (3)}}\\
&= 0
\end{align*}
By convexity, the adjuster solution that adjusts the baseline to the adversary model minimizes the adjuster loss. Note that essential to this proof is that the derivatives f(x) and g(x) are not functions of their parameter $\beta$ and so evaluate to the same value.
\end{proof}

Finally, we provide a bound on the size of the adjustment in terms of the increase in fairenss as measured by $L_a$.
\begin{prop}[Accuracy-Fairness Trade-Off]
\begin{align*}
\text{Adjuster Objective} &: \|g\|^2 + L_a\!\bigl(\hat{y} + g\bigr) \\
\text{At }g=0&:\|0\|^2 + L_a(\hat{y}) = L_a(\hat{y}) \\
\text{At }g\neq 0&:\|g\|^2 + L_a(\hat{y} + g) \\
\text{If }g^* \neq 0 &: \|g^*\|^2 + L_a(\hat{y}+g^*) \;\le\; L_a(\hat{y}) \\
&\Rightarrow \|g^*\|^2 \;\le\; L_a(\hat{y}) - L_a(\hat{y}+g^*) 
\end{align*}
Thus, the nonzero adjuster causes a decrease in $L_a\text{ at least }\|g\|^2$.
\end{prop}

This is helpful when our fairness loss is somewhat interpretable. Consider a fairness penalty that attempts to minimize the difference in overprediction error between two partitions of our data $G_1$ and $G_2$

Define the \emph{overprediction} in group $k$ by the average $\frac{1}{|G_k|}\sum_{i\in G_k} \bigl(f_i - y_i\bigr)$. 

\[
L_a(f) 
=
\frac{1}{|G_1|}\!\sum_{i\in G_1}(f_i-y_i)
-\frac{1}{|G_2|}\!\sum_{j\in G_2}(f_j-y_j).
\]

\[
L_a(f+g)
= 
\frac{1}{|G_1|}\sum_{i\in G_1}(f_i + g_i - y_i)
-
\frac{1}{|G_2|}\sum_{j\in G_2}(f_j + g_j - y_j).
\]

Due to shared terms, $f$ and $y$ cancel out in the difference:
\[
L_a(f+g) - L_a(f) 
=
\Bigl[\frac{1}{|G_1|}\sum_{i\in G_1} g_i\Bigr]
-\Bigl[\frac{1}{|G_2|}\sum_{j\in G_2} g_j\Bigr]
\geq ||g||^2
\]
Therefore, the \emph{adjustment} $g$ shifts the overprediction difference at least by the average squared offset, which can be a useful property in model selection. \\

\subsection{Binary Classification}
In the classification setting, we are able to characterize the difference in model performance between adversarial debiasing and the fairness adjuster analogously to the regression case.
In the classification setting, the three objectives are:

\begin{itemize}
\item \textbf{Baseline:}
\[
L_{B}(f)
\;=\;
-\sum_{i=1}^n 
\Bigl[
  y_i\,\ln\!\bigl(\sigma(f(x_i))\bigr)
  \;+\;
  \bigl(1-y_i\bigr)\,\ln\!\bigl(1-\sigma(f(x_i))\bigr)
\Bigr].
\]
Here, $f(x)$ is the baseline logit; $\sigma(z) = 1/(1+e^{-z})$.  

\item \textbf{Adversarial Debiasing:}
\[
L_{AD}(h)
\;=\;
-\sum_{i=1}^n 
\Bigl[
  y_i\,\ln\!\bigl(\sigma(h(x_i))\bigr)
  +
  \bigl(1-y_i\bigr)\,\ln\!\bigl(1-\sigma(h(x_i))\bigr)
\Bigr]
\;+\;
\lambda L_a\bigl(h(X)\bigr).
\]
Logit $h(x)$ now aims to reduce both cross-entropy and a fairness penalty $L_a(\cdot)$.  

\item \textbf{Adjuster:}
Let $\hat{y}_i = \sigma\bigl(f(x_i)\bigr)$. Then
\begin{align*}
L_{O}\bigl(f + g\bigr) \;&=\; -\sum_{i=1}^n
\Bigl[
  \hat{y}_i\,\ln\!\bigl(\sigma(f(x_i) + g(x_i))\bigr)
   \;+\;
  \bigl(1-\hat{y}_i\bigr)\,\ln\!\bigl(1-\sigma(f(x_i) + g(x_i))\bigr)
\Bigr] \\
& \;\;\;\,\,\;+\;
\lambda L_a\!\bigl(f(X) + g(X)\bigr).
\end{align*}
Here, $\hat{y}_i$ is treated as a pseudo-label, and $g(x)$ is an adjuster to the baseline logit $f(x)$.
\end{itemize}

\begin{prop}[Bounding Cross-Entropy for Adjuster vs. Adversarial Debiasing]\label{prop:CEbound}
In the setting above, make the same assumptions as \ref{prop:MSEbound}. Then, the following bound on cross-entropy with respect to the true labels \(Y\) holds:
\[
  \mathrm{BCE}(f + g, Y)
  \le
  \mathrm{BCE}(h, Y)
  +
  (\sigma(f) - Y)^\top\bigl[g - (h - f)\bigr].
\]
\end{prop}

\begin{proof}
We first recall a convenient identity for the cross-entropy expansion. For any vector 
\(u = (u_1,\dots,u_n)\) and a label vector \(\hat{y} = (\hat{y}_1,\dots,\hat{y}_n)\), we write:
\[
  \mathrm{BCE}(u,\hat{y})
  =
  \mathrm{BCE}(u,y) 
  -
  \sum_{i=1}^n \bigl[\sigma(f_i) - y_i\bigr]\,\log\!\frac{\sigma(u_i)}{1 - \sigma(u_i)}.
\]
Since \(\log\bigl(\tfrac{\sigma(u_i)}{1-\sigma(u_i)}\bigr) = u_i\), this becomes
\[
  \mathrm{BCE}(u,\hat{y})
  =
  \mathrm{BCE}(u,y)
  -
  \sum_{i=1}^n \bigl[\sigma(f_i) - y_i\bigr]\,u_i.
\]

By assumption of optimality of the adjuster solution \(f + g\), we have
\[
  \mathrm{BCE}(f + g, \hat{y}) + \lambda L_a(f + g)
  \le
  \mathrm{BCE}(h, \hat{y}) + \lambda L_a(h).
\]
Given \(L_a(f + g) = L_a(h)\), this implies
\[
  \mathrm{BCE}(f + g, \hat{y})
  \le
  \mathrm{BCE}(h, \hat{y}).
\]

We next apply the cross-entropy identity to both \(\mathrm{BCE}(f + g, \hat{y})\) and 
\(\mathrm{BCE}(h, \hat{y})\):
\[
  \mathrm{BCE}(f + g, \hat{y})
  =
  \mathrm{BCE}(f + g, y)
  -
  \sum_{i=1}^n \bigl[\sigma(f_i) - y_i\bigr]\bigl(f_i + g_i\bigr),
\]
\[
  \mathrm{BCE}(h, \hat{y})
  =
  \mathrm{BCE}(h, y)
  -
  \sum_{i=1}^n \bigl[\sigma(f_i) - y_i\bigr]\,h_i.
\]
Since \(\mathrm{BCE}(f + g, \hat{y}) \le \mathrm{BCE}(h, \hat{y})\), we rearrange to get
\[
  \mathrm{BCE}(f + g, y)
  -
  \sum_{i=1}^n \bigl[\sigma(f_i)-y_i\bigr]\bigl(f_i + g_i\bigr)
  \le
  \mathrm{BCE}(h, y)
  -
  \sum_{i=1}^n \bigl[\sigma(f_i)-y_i\bigr]\,h_i.
\]
Equivalently,
\[
  \mathrm{BCE}(f + g, y)
  \le \mathrm{BCE}(h, y)
  +
  \sum_{i=1}^n \bigl[\sigma(f_i) - y_i\bigr]\bigl(g_i - (h_i - f_i)\bigr).
\]
This can be written in vector form as
\[
  \mathrm{BCE}(f + g, y)
  \le
  \mathrm{BCE}(h, y)
  +
  (\sigma(f) - y)^\top\bigl[g - (h - f)\bigr].
\]
\end{proof}

Notice that this mirrors the MSE bound \footnote{
Interestingly, we can write both cases as $L(f + g) \leq L(h) + (\nabla_f L)^T (g - (h - f))$. It's unclear whether this is a coincidence or indicative of deeper structure.
}:
\[
L_B(f+g) \;\le\; L_B(h) \;+\; 2(f-y)^\top\bigl[g-(h-f)\bigr].
\]

\section{Experiments}

In this section, we present experimental evaluation of the proposed methodology on three datasets: the Adult \cite{uci}, COMPAS \cite{compas}, German credit \cite{uci}.
We implement the method using a combination of XGBoost as the classifier and a simple logistic regression as the adversary \footnote{
	Gradient boosted decision trees, and XGBoost in particular, have been shown to achieve more favorable performance than Neural Networks on tabular data, especially on skewed data \cite{mcelfresh2024neuralnetsoutperformboosted}.
	On the other hand, NNs are more flexible than gradient boosted tree methods, in which a tree is fixed once it is generated, since they are able to update all parameters throughout the training process. Since the input feature of the NN is a changing output, a GBT method would not be well-suited as the adversary.
}.
For each dataset, we ran a Bayesian hyperparameter search \footnote{
	Our implementation of adversarial debiasing relies on training the XGBoost model for more than 10 iterations.
	In cases where the optimal hyperparameters were fewer than 10 iterations, we multiplied the optimal number of iterations by a constant $n$ and scaled the learning rate by $1/n$ to compensate. All three models (baseline, adversarial debiasing and adjuster) were then trained with the same hyperparameters.
} using the FLAML package \cite{wang2021flaml} and then computed fairness metrics using the IBM's AIF360 package \cite{aif360-oct-2018}.
We first describe the experimental setup and subsequently provide the experimental results. 

\subsection{Experimental Setup}
We evaluated the \textbf{fairness adjuster} using the following datasets: 

\begin{list}{}
\item \textbf{Adult} This dataset contains features, demographic information, and a binary target encoding if income is greater than \$50k for 48k individuals. We treated sex as the protected attribute.
\item \textbf{German} This dataset contains features, demographic information, and a binary  target encoding good or bad credit risk for 1000 individuals. We treat age ($\geq$ 26 years) as the protected attribute.
\item \textbf{COMPAS} This dataset contains features, demographic information, and a binary target encoding if an individual re-offended within 2 years. We treated race as the protected attribute. \\
\end{list}

For each dataset we fit a baseline learner to establish the fairness of a model trained solely to optimize classification accuracy.
Each dataset is a classification task, so the baseline model minimizes the BCE (binary cross-entropy) loss.
We executed the same preprocessing procedure as implemented in the AIF360 package on the raw feature values.
In order to provide robust evaluation metrics we used 50 random seeds to execute a 5-fold cross validation routine. Within each fold and seed combination, we fit following algorithms: 

\begin{list}{}
\item \textbf{Baseline} XGBoost binary classifier, optimized without any fairness constraint. 
\item \textbf{Adversarial Debiasing} XGBoost binary classifier, optimized with the fairness loss constraint in the objective function.
This provided the main comparison between the proposed methodology and an existing, in-processing alternative algorithm to achieve demographic parity as shown in \cite{ad-xgb}.
\item \textbf{Fairness Adjuster} We implemented the methodology proposed in this paper using the baseline XGBoost classifier in combination with an XGBoost regressor issuing the adjustments to the classification predictions.\\
\end{list}

Within each seed and fold combination, we evaluated each of the above algorithms both Classifications Accuracy and Disparate Impact defined as follows:
\begin{list}{}
\item \textbf{Classification Accuracy} The proportion of correct classifications issued by the model $\frac{\sum_{i=1}^n \hat{y} = y}{n}$.
\item \textbf{Disparate Impact} defined as the proportion of favorable outcomes predicted for the protected group divided by the proportion of favorable outcomes predicted for the non-protected group.
For each of the datasets, we can clearly define the favorable outcome as
\begin{list}{}
\item \textbf{Adult} High income
\item \textbf{German} Good credit risk
\item \textbf{COMPAS} Low recidivism risk\\
\end{list}
\end{list}

Finally, through a linear search, we found the corresponding weights for the adversary and adjuster algorithms which approximately achieved equalized odds \cite{equalopportunity} for the protected and non-protected groups.

\subsection{Results}

For all data sets tested, we found that the \textbf{fairness adjuster} achieved nearly equivalent accuracy to adversarial debiasing when set to achieve equalized odds (i.e. Disparate Impact = 1).

To compare between methodologies, we calculated the mean Classification Accuracy and Disparate Impact (called ``Fairness" in the table) scores across all folds and constructed a 95\% confidence interval about the mean. 

Finally, we calculated the differences between the adversarial debiasing and the fairness adjuster loss functions $\Delta \text{Loss} = \frac{1}{n} (\sigma(f) - y)^T (g - (h - f))$\footnote{
    We add a factor of $\frac{1}{n}$ to control for the number of observations in each dataset.
} derived above. We find these differences to be close to 0 for all datasets, and consequently the adversarial debiasing and adjuster models have similar performance.

\begin{center}
\begin{tabular}{||c | c c c ||} 
	\hline
	Dataset & German & Adult & Compas \\ [0.5ex] 
	\hline\hline
	Baseline Accuracy & 0.7055±0.0011 & 0.8036±0.0001 & 0.6629±0.0005 \\ 
	\hline
	AD Accuracy & 0.6980±0.0004 & 0.7802±0.0007 & 0.6588±0.0005 \\ 
	\hline
	Adjuster Accuracy & 0.6977±0.0004 & 0.7795±0.0006 & 0.6584±0.0005 \\ 
	\hline\hline
	Baseline Fairness & 0.8215±0.0098 & 0.0000±0.0000 & 0.6999±0.0032 \\ 
	\hline
	AD Fairness & 1.0027±0.0009 & 1.0337±0.0232 & 1.0004±0.0110 \\ 
	\hline
	Adjuster Fairness & 1.0040±0.0006 & 1.0466±0.0246 & 1.0098±0.0100 \\ 
	\hline\hline
	$\Delta$ Loss & 0.00071 & $4.5045 \times10^{-5}$ & -0.00090 \\ [0.5ex] 
	\hline
\end{tabular}
\end{center}

\begin{figure}[h!]
    \centering
    \begin{subfigure}[b]{0.4\textwidth}
        \centering
        \includegraphics[width=\textwidth]{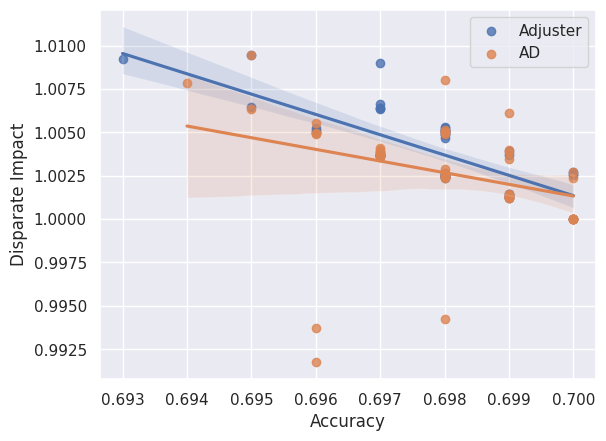}
        \caption{German dataset}
        \label{fig:german}
    \end{subfigure}
    \begin{subfigure}[b]{0.4\textwidth}
        \centering
        \includegraphics[width=\textwidth]{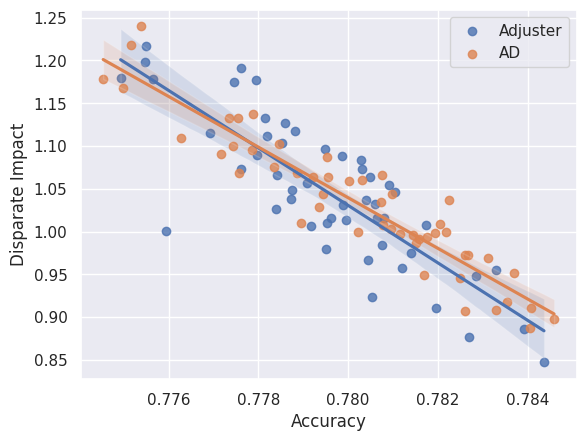}
        \caption{Adult dataset}
        \label{fig:adult}
    \end{subfigure}
    \begin{subfigure}[b]{0.4\textwidth}
        \centering
        \includegraphics[width=\textwidth]{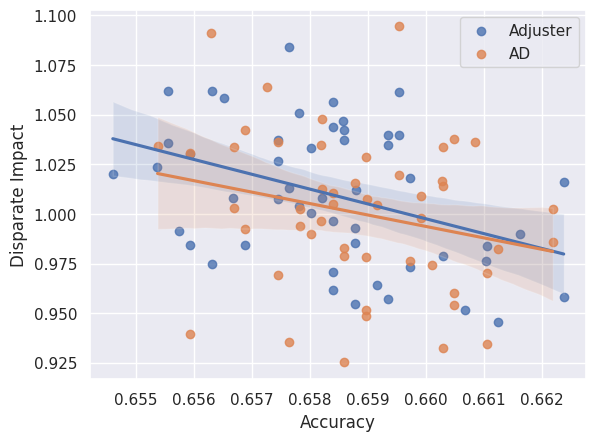}
        \caption{COMPAS dataset}
        \label{fig:compas}
    \end{subfigure}

    \caption{
    	50-seed, 5-fold CV results for three datasets: German, Adult, and COMPAS.
		The figures show the fairness-accuracy tradeoff for Adversarial Debiasing and the fairness adjuster resulting from the random CV splitting.
		Each point represents the average of the metrics over the 5 folds for either method.
		We fit a regression line with confidence bounds to better visualize the results.
    }
    \label{fig:results}
\end{figure}

\section{Discussion}
The fairness adjuster method has the following advantages:

\begin{list}{}
\item \textbf{Decoupling From Model Implementation} Our method eliminates the need to modify the loss function of the underlying model.
In contrast, competing approaches often require manual implementation of custom loss functions, which is both technically challenging and highly context-specific.
\item \textbf{Flexibility in Training Data} Unlike traditional methods that require both demographic data and training labels on the training data to optimize fairness, our approach demonstrates that training labels for the dataset used to tune fairness are not necessary.
This flexibility allows for the base model to be trained on one dataset and subsequently fine-tuned for fairness using a different dataset, even one for which demographic data is available, but targets lacking.
\item \textbf{Compatibility with Complex or Proprietary Models} The framework supports fairness tuning of complex or proprietary base models by treating them as black-box systems.
This enables the adjustment of more complex base models for which implementing in-processing methods would be difficult or impossible, e.g. an ensemble of multiple models.
Furthermore, our method allows separate entities to independently optimize models for accuracy and fairness, potentially reducing conflicts of interest and facilitating post-hoc fairness adjustments without altering the original model.
\item \textbf{Enhanced Interpretability} The fairness adjuster provides greater transparency into the effects of fairness adjustments at the individual level.
Unlike combined models that output a single prediction, the adjustment-based approach produces both the original prediction and the adjustment applied to that prediction.
This separation enables more granular interpretability and control over the final model, such as identifying fairness-relevant features or imposing constraints on the adjustment.
\end{list}

\section{Future Work}
While our current framework focuses on a single fairness constraint based on adversarial debiasing, other directions warrant further exploration:
\begin{enumerate}
\item \textbf{Adapting to Additional Fairness Metrics.} We have primarily demonstrated our method with adversarial loss functions or group-based overprediction metrics. Future work could incorporate alternative definitions such as Equalized Odds, Calibration, or other constraints that may require specialized loss surrogates. A challenge is ensuring that the pseudo-label choice still admits an offset solution with minimal accuracy degradation.
\item \textbf{Extending to Complex Model Architectures.} Our offset approach can be applied to models with structured outputs (e.g.\ sequences) or to more intricate neural network structures. Investigating how to design the “adjuster” component for architectures like large language models or deep generative models could open new avenues of post‐hoc fairness tuning.
\item \textbf{Robustness and Generalization Guarantees.} While we outlined theoretical results comparing in‐processing vs. offset solutions, more rigorous bounds on generalization error under fairness constraints remain an open question. For instance, combining Rademacher complexity arguments with fairness constraints might yield improved theoretical guarantees for large‐scale offset solutions.
\item \textbf{Handling Multiple Protected Attributes.} Much of the literature, including our main demonstration, focuses on one sensitive attribute at a time (e.g.\ race, gender). Future work could investigate how to incorporate additional attributes or intersectional fairness constraints, ensuring that our offset framework remains tractable.
\end{enumerate}

\section{Conclusion}
We introduced a \emph{post‐processing} framework, the \textbf{fairness adjuster}, that adapts in‐processing techniques for use on top of an existing model’s predictions. Our theoretical analysis showed that, under certain assumptions, this offset approach can reproduce a similar solution as a directly debiased model, at a minimal accuracy loss for a comparable fairness improvement. Empirical results on real datasets (Adult, German, and COMPAS) demonstrated that our fairness adjuster achieves a near‐identical fairness-accuracy tradeoff to the in‐processing adversarial debiasing method, while offering important practical benefits: the elimination of custom model‐training constraints, separation of the baseline model’s optimization from fairness constraints, and interpretability of the offset itself. By treating models as black boxes, it also provides flexibility in datasets used for fairness tuning and accommodates proprietary systems. We hope that this methodology fosters a broader adoption of fairness improvements with minimal impact on existing machine learning pipelines.

\section{References}
\bibliographystyle{ACM-Reference-Format}
\bibliography{paper/fairness_adjuster_paper}

\end{document}